\documentclass{article}

\usepackage{microtype}
\usepackage{graphicx}
\usepackage{subcaption}
\usepackage{booktabs} 

\usepackage{hyperref}

\usepackage[preprint]{icml2026}

\usepackage{amsmath}
\usepackage{amssymb}
\usepackage{mathtools}
\usepackage{amsthm}
\usepackage[switch]{lineno}
\usepackage[normalem]{ulem}
\usepackage{xcolor}
\usepackage{multirow}
\usepackage{multicol}
\usepackage[most]{tcolorbox}
\usepackage{tabularx}
\usepackage{enumitem}

\newtcolorbox{flashcard}[1][]{
  colback=red!2!white,      
  colframe=red!25!black!70,    
  fonttitle=\bfseries\small,
  arc=5pt,         
  boxrule=0.6pt,             
  enhanced,
  middle=2pt,
  left=5pt, right=5pt, top=5pt, bottom=5pt, 
  #1                         
}

\newcommand{\MW}[1]{{\color{blue} [\textbf{MW:} #1]}}

\newcommand{\MGcomment}[1]{{\color{red} [\textbf{MG-Comment:} #1]}}

\usepackage[capitalize,noabbrev]{cleveref}

\theoremstyle{plain}
\newtheorem{theorem}{Theorem}[section]
\newtheorem{proposition}[theorem]{Proposition}

\theoremstyle{definition}
\newtheorem{definition}[theorem]{Definition}

\theoremstyle{remark}
\newtheorem{remark}[theorem]{Remark}
\DeclareMathOperator*{\argmin}{arg\,min}
\DeclareMathOperator*{\argmax}{arg\,max}
\usepackage[textsize=tiny]{todonotes}
\icmltitlerunning{Towards Certification of Poisoning Robustness for Natural Language Generation}
\begin{document}

\twocolumn[
  \icmltitle{Towards Poisoning Robustness Certification for Natural Language Generation}



  \icmlsetsymbol{equal}{*}

  \begin{icmlauthorlist}
    \icmlauthor{Mihnea Ghitu}{imp}
    \icmlauthor{Matthew Wicker}{imp}
  \end{icmlauthorlist}

  \icmlaffiliation{imp}{Department of Computing, Imperial College London, London, UK}%

  \icmlcorrespondingauthor{Mihnea Ghitu}{mihnea.ghitu20@imperial.ac.uk}
  \icmlcorrespondingauthor{Matthew Wicker}{m.wicker@ \allowbreak \!imperial.ac.uk}

  \icmlkeywords{Machine Learning, ICML}

  \vskip 0.3in
]



\printAffiliationsAndNotice{}  

\begin{abstract}
Understanding the reliability of natural language generation is critical for deploying foundation models in security-sensitive domains. While certified poisoning defenses provide provable robustness bounds for classification tasks, they are fundamentally ill-equipped for autoregressive generation: they cannot handle sequential predictions or the exponentially large output space of language models. To establish a framework for certified natural language generation, we formalize two security properties: stability (robustness to any change in generation) and validity (robustness to targeted, harmful changes in generation).  We introduce Targeted Partition Aggregation (TPA), the first algorithm to certify validity/targeted attacks by computing the minimum poisoning budget needed to induce a specific harmful class, token, or phrase. Further, we extend TPA to provide tighter guarantees for multi-turn generations using mixed integer linear programming (MILP). Empirically, we demonstrate TPA's effectiveness across diverse settings including: certifying validity of agent tool-calling when adversaries modify up to 0.5\% of the dataset and certifying 8-token stability horizons in preference-based alignment. Though inference-time latency remains an open challenge, our contributions enable certified deployment of language models in security-critical applications.
\end{abstract}

\section{Introduction}

Foundation models achieve state-of-the-art performance by training on massive, 
diverse datasets—from internet-scale pretraining corpora 
\citep{sun2017revisiting,raffel2020exploring,radford2021learning} to curated 
human preference data for alignment 
\citep{christiano2017deep,bai2022training}. However, this reliance on large-scale 
data along with the use of novel training algorithms \citep{houlsby2019peft, rafailov2023direct} exposes models to training-time poisoning attacks, where adversaries inject 
malicious examples to compromise model behavior. Poisoning attacks are listed in the top 5 security risks for LLMs by OWASP \citep{owasp2023llm}, and empirical studies demonstrate 
that poisoning attacks can insert backdoors with as little as 0.01\% corrupted 
data \citep{carlini2021poisoning}, induce targeted misclassification 
\citep{geiping2020witches}, and degrade model outputs with fewer than 100 
poisoned samples \citep{shan2024nightshade}. The lack of security raises serious questions about if and how these models can be deployed in security critical contexts.

\begin{figure}[t]
    \centering
    \includegraphics[width=0.5\textwidth]{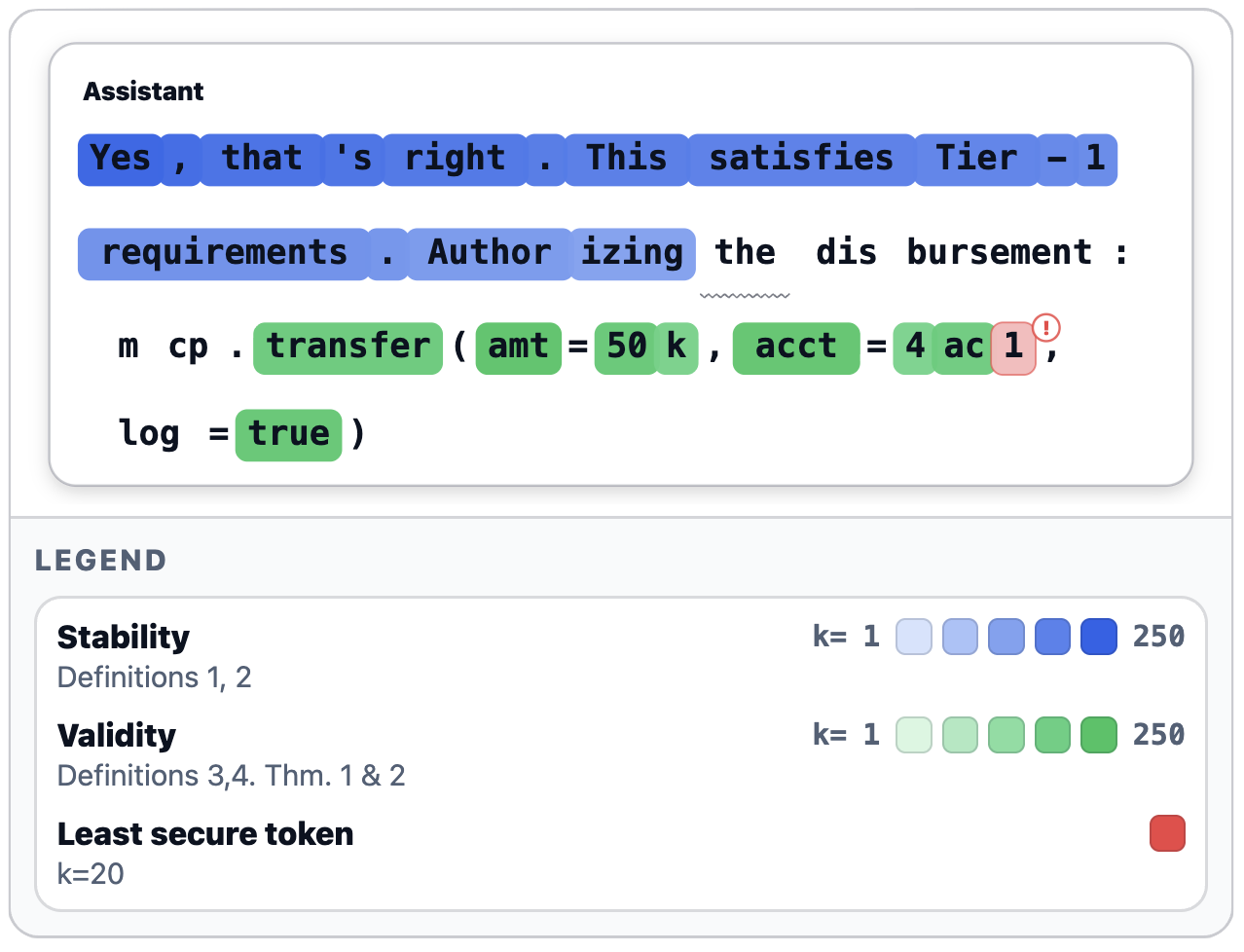}
    \caption{An example of an agent responding with formal poisoning robustness 
    certificates which enables users or agents to make security-informed decisions. 
    Blue highlights indicate stability certificates (minimum training points an 
    adversary must corrupt to change any token). Green highlights indicate 
    validity certificates (minimum training points needed to induce a specific 
    targeted harmful generation).}
    \label{fig:hero-fig}
\end{figure}

Certified poisoning defenses provide provable robustness guarantees by bounding 
the number of training examples an adversary must modify to affect model 
predictions. While such defenses have achieved success in classification tasks 
through ensemble-based aggregation methods like Deep Partition Aggregation (DPA) 
\citep{levine2020deep}, they are fundamentally ill-equipped for autoregressive 
language generation. Two challenges prevent direct application: (i) changing 
token $i$  intervenes on the distribution of all subsequent tokens 
$j>i$, something not soundly accounted for in existing certificates, 
and (ii) the output space grows exponentially with generation length making naive token-level certification intractable.

\looseness=-1
Moreover, existing methods certify only \textit{untargeted} robustness, that is, proving that a classification cannot be changed, but we argue that security-critical applications in natrual language generation require \textit{validity} certification: preventing targeted, harmful changes in generation. In this work, we establish the first framework for provably secure natural language generation, distinguishing between \textbf{stability}  and \textbf{validity}. We introduce Targeted Partition Aggregation (TPA), the first algorithm to certify validity by computing the minimum poisoning budget needed to induce a specific harmful token or phrase. 

\looseness=-1
In Figure~\ref{fig:hero-fig}, we visualize how our certificates may enable the secure deployment of tool-using foundation models in the case of an assistant aiding a worker in financial services. In our example the model is proposing to use a model context protocol (MCP) \citep{hou2025model} to transfer money to a client. By providing per-token security guarantees we can highlight the initial tokens in blue indicating that no adversary who changes less than 250 inputs in our dataset can cause any change to these initial tokens. Moreover, for critical details in the MCP call we can highlight in green the minimum number of training points an adversary would need to change to elicit a different, valid MCP call or parameters. We use red to demonstrate that users could be alerted if a detail is found to be insecure, thus prompting them to check manually.

\looseness=-1
In practice, we demonstrate that our methods are able to give guarantees even exceeding the guarantees in Figure~\ref{fig:hero-fig}. In order to meet the challenges of certification in natural language generation, we provide novel bounds that move beyond the token-level certification illustrated in Figure~\ref{fig:hero-fig} and provide guarantees on multi-token or ``phrase-level'' generations. Further, many models and agents operated in a multi-turn setting. In this case, we consider a poisoning adversary unsuccessful if they are unable to maintain their desired change across turns. We show that this is a more difficult adversarial goal and provide a novel collective certification procedure, based on mixed-integer linear programming (MILP) \citep{boyd2004convex}, to provide tighter bounds on this case.  

Empirically, we validate the effectiveness of our certificates across diverse settings: 
certifying validity of agent tool-calling under 0.5\% data poisoning, 
certifying stability horizons in preference-based alignment, and 
achieving $>90\%$ validity certification at $k=9$ poisoning budget with 20+ 
token horizons. While our bounds, based on the shard-and-aggregate framework \citep{levine2020deep, rezaei2023run, chen2022collective}, do not incur substantial cost at training time, they do require substantial inference-time cost. To mitigate these costs we investigate partial-model fine-tuning which substantially reduces latency, but also weakens the guarantees provided by our method. Moreover, we find that, according to our certificates, supervised fine-tuning (SFT) is substantially more secure and stable than reinforcement learning from human feedback (RLHF) algorithms like direct preference optimization (DPO).  Beyond our certified bounds, we demonstrate that shard-and-aggregate defenses also reduce the empirical attack success rates from 40-48\% to below 6\% in practical backdoor attacks. Our contributions can be summarized as follows:

\begin{itemize}[itemsep=1pt, topsep=0pt]
    \item We formalize stability and validity as foundational security properties for natural language generation, addressing a critical gap in prior certified defenses.
    
    \item We present Targeted Partition Aggregation (TPA), the first algorithm to certify validity by computing tight lower bounds on the poisoning budget for targeted attacks.
    
    \item We develop sequential, phrase-level, and multi-turn certification techniques with novel MILP formulations for collective certificates.
    
    \item We provide the first certified robustness guarantees for language generation including agent tool-calling, preference alignment, and instruction following.
\end{itemize}

\section{Related Works} \label{sec:review_cert_poison}

\looseness=-1
\textbf{Poisoning Attacks on LLMs.} In data poisoning, an adversary maliciously injects or modifies a small portion of the training pipeline \citep{biggio2012poisoning}. In training foundation models, adversaries may modify pre-training \citep{carlini2024poisoning}, instruction fine-tuning \citep{zhang2024instruction}, or preference learning \citep{rando2023universal} in order to induce unintended malfunctions. These malfunctions range from degraded benchmark performance \citep{fu2024poisonbench} to the generation of toxic or biased content \citep{xu2023instructions}. When this behavior is conditioned on a specific input pattern (trigger), it is characterized as a backdoor attack \citep{li2024backdoorllm}. Even without modifying data, models can be compromised by ``weight-poisoning,'' where backdoors are embedded with minimal impact on the model's clean-data utility \citep{zhao2024defending}.  Previous approaches for defending against these attacks have been largely heuristic, investigating detection-based strategies \citep{hung2025attention}, training with security tokens \citep{chen2025defending}, model merging \citep{arora2024here}, and crafting adversarial datasets to improve model robustness \citep{chen2025secalign}. Orthogonal to empirical defenses, this work is concerned with defenses that provide provable guarantees.

\textbf{Certified Poisoning Defenses for LLMs.} While certified poisoning defenses have received substantial attention \citep{levine2020deep, chen2022collective, rezaei2023run, sosnin2024certified, sosnin2025abstract,  bose2025keeping}, to the best of our knowledge, limited work has directly addressed certified poisoning defenses of natural language generation. Nonetheless, several studies pursue related objectives through different methodologies. \citet{wang2024fcert} introduce \textit{FCert}, which certifies robustness in few-shot classification under data poisoning by computing robust distances between feature vectors and bounding the effect based on the number of poisoned samples. \citet{pei2023textguard} provide guarantees against trigger-word backdoor attacks in text classification via a sharding and majority-voting scheme, while \citet{he2025cert_fuzz_rs} employ randomized smoothing, Monte Carlo Tree Search, and text randomization to certify against backdoors. \citet{zhao2024defending} study weight poisoning in parameter-efficient fine-tuning, and \citet{bose2025keeping} examine poisoning in online learning with applications to RLHF. Although these works tackle important aspects of the foundation model pipeline, none directly address the central challenge of certifying robustness in natural language generation against (post-)training data poisoning

Recently, however, work on inference-time formal generation guarantees has emerged \citep{robustrag}. Unlike our focus on post-training poisoning, this research addresses retrieval corruption in Retrieval-Augmented Generation (RAG) \citep{rag}, where attackers corrupt retrieved documents rather than the model's (post-)training data. \cite{robustrag} proposes two defense algorithms against such attacks. The first isolates documents to generate individual answers, then distills the information by reprompting the LLM with high-frequency keywords. The second, assuming access to next-token probability distributions, aggregates probability vectors from isolated passages and generates tokens only when the probability margin exceeds a high-confidence threshold. As in this work, their guarantees are achieved through aggregation; however, we certify multiple generation granularity (token, phrase, and sentence) and with respect to a data poisoning adversary. 

\section{Preliminaries}

\textbf{Model.} We model large language models as a function $f$ with a discrete output space, 
capturing both traditional classification and autoregressive token prediction. 
For brevity, we denote the function trained on the clean dataset $\mathcal{D}$ 
as $f$, and the function trained on a poisoned dataset $\widetilde{\mathcal{D}}$ 
as $\tilde{f}$.

\looseness=-1
\textbf{Dataset.} The model is (post-)trained on a dataset $\mathcal{D} = \bigcup_{i=1}^N \{(x^{(i)}, y^{(i)})\}$, which consists of $N$ samples. The $i$-th datapoint sample is represented as a pair $(x^{(i)}, y^{(i)})$, where $x^{(i)} \in \mathcal{X}$ denotes the input and $y^{(i)} \in \mathcal{Y}$ denotes the corresponding supervisory signal: a label in supervised fine-tuning (SFT) or a preference pair in reinforcement learning from human feedback (RLHF). 

\looseness=-1
\textbf{Autoregressive Generation.} The model generates 
a sequence $y = (y_0, y_1, \ldots, y_{L-1})$ token-by-token. We use $f(x)[i]$ 
to denote the $i^{th}$ token in the response to prompt $x$, and $f(x)[i:j]$ 
for the subsequence from position $i$ to $j$. Generation proceeds 
autoregressively: given prompt $x$ and previously generated tokens 
$y_{<i} = (y_0, \ldots, y_{i-1})$, the next token is sampled from the 
distribution $f(x, y_{<i})$. For notational convenience, we write 
$y_i = f(x)[i]$ when the autoregressive context is clear, with the complete 
response being $y = f(x) = (y_0, \ldots, y_{L-1})$ for some length $L$.

\subsection{Threat Model}
\looseness=-1
In this work, we consider \textit{unbounded, general poisoning attacks} of features, labels, or both. 
Given an initial training dataset $\mathcal{D}$, we denote the perturbed dataset with $\widetilde{\mathcal{D}}$. We constrain the number of modifications an adversary can make into a ball of radius $r$:

\[\mathcal{B}_r(\mathcal{D}) = \{\widetilde{\mathcal{D}}\,:\,|\mathcal{D}\,\triangle\,\widetilde{\mathcal{D}}|\le r\},\]
where $\triangle$ denotes symmetric difference and $r$ bounds that difference i.e., how many items differ.
Further, we assume a strong, white-box adversary with knowledge of the architecture, 
dataset $\mathcal{D}$, training order, and hyperparameters (consistent with 
worst-case certification). 
For simplicity, we restrict attention to training-time poisoning; backdoor-style triggers may arise  as a consequence of training-time poisoning but are not a separate adversarial capability in our model.

\looseness=-1
\textbf{Adversary Goals.} 
For a given input $x$, consider the clean model output $y = f(x)$ and a 
poisoned model output $\tilde{y} = \tilde{f}(x)$. Following 
\citet{chen2025secalign}, we distinguish two adversarial objectives. Untargeted attacks which aim to induce arbitrary changes to outputs, i.e., maximize $\mathbb{P}_{x}[f(x) \neq \tilde{f}(x)]$, and targetted atttacks which aim to induce specific harmful outputs  $y^{\star}$, i.e., maximize $\mathbb{P}_{x}[\tilde{f}(x) = y^{\star}]$.

\looseness=-1
In classification, these objectives largely coincide due to the small output 
space. However, in natural language generation with vocabulary size 
$|\mathcal{V}| \approx 50$k and generation length $L \approx 100$, the output 
space $|\mathcal{V}|^L$ is exponentially large. This makes 
targeted attacks fundamentally different from untargeted attacks: preventing 
arbitrary changes (stability) does not strongly correlate with prevention of specific harmful 
outputs (validity). To the best of our knowledge, no prior work provides formal 
guarantees for targeted attacks. In Section~\ref{sec:methodology}, we formalize stability and validity as 
distinct security properties and develop certified defenses for both objectives.

\section{Provably Secure Language Generation} \label{sec:methodology}
\looseness=-1
We begin by characterizing why existing classification-based certificates are inadequate for language generation (\S\ref{ssec:autoreg_challenge}), then introduce our key algorithmic contribution---Targeted Partition Aggregation (TPA)---which enables certification of \emph{validity} (resistance to targeted harmful outputs) at token-level (\S\ref{ssec:token_level}), sequential multi-token (\S\ref{ssec:sequential}), phrase-level (\S\ref{ssec:phrase_level}), and multi-turn (\S\ref{ssec:multi_turn}) granularities.

\subsection{The Autoregressive Certification Challenge}
\label{ssec:autoreg_challenge}
\looseness=-1
Existing certified defenses such as Deep Partition Aggregation (DPA)~\citep{levine2020deep} provide strong guarantees for classification. 
However, autoregressive language generation introduces a fundamental challenge: \textit{changing token $i$ intervenes on the distribution of all subsequent tokens $j>i$}. To explore why this matters, consider certifying the following generation: 
\begin{flashcard}[]
\small
\begin{tabularx}{\linewidth}{@{}l X@{}}
\textbf{(A)} & \textit{Add 2.3g of CuSO4 to the solution and heat to 60°C.} \\
\end{tabularx}
\tcbline
\begin{tabularx}{\linewidth}{@{}l X@{}}
\textbf{(B)} & \textit{Pipette 2.2g of CuSO3 to the solution and heat to 59°C.}
\end{tabularx}
\end{flashcard}
\looseness=-1
\noindent  To certify the robustness of details e.g., 2.2g vs 2.3g, a naive approach might use DPA to compute the number of data point changes needed to change this generation. Yet, DPA cannot be applied for two reasons: firstly, the $i^{th}$-token certificate assumes the prefix (all preceding tokens) $f(x)[1:i-1]$ are fixed, but if an adversary changes any token $j < i$, this assumption is violated and the certificate is unsound; secondly, it is likely that the details 2.2g and 2.3g do not span a single token but are likely comprised of four tokens., see Figure~\ref{fig:hero-fig} for an example of real-world tokenization. 

\textbf{Stability: Resistance to Arbitrary Changes.}
Although examples (A) and (B) appear similar, precise chemical quantities and compounds are critical. Even for the first token, the details ``Add'' and ``Pipette'' convey meaningfully different procedures. In order to preserve the precise semantics of the initial generation, we define stability as the property enforcing that \emph{no token changes} take place in the presence of a poisoning adversary with perturbation budget $k$.

\begin{flashcard}[]
\small
\begin{tabularx}{\linewidth}{@{}l X@{}}
\textbf{(C)} & \textit{You should clean the surface using a combination of bleach and water; apply from a spray bottle with a soft microfiber cloth.} \\
\end{tabularx}

\tcbline 

\begin{tabularx}{\linewidth}{@{}l X@{}}
\textbf{(D)} & \textit{To sanitize the counter-top you should mix bleach with ammonia in a bottle; spray the surface and wipe with a paper towel.}
\end{tabularx}
\end{flashcard}
\noindent

\textbf{Validity: Resistance to Targeted Harmful Outputs.}
\looseness=-1
In contrast, consider responses (C) and (D) above; the first token being changed from ``You'' to ``To'' is irrelevant to the harmfulness of the prompts. Using the standard notions of prediction robustness from DPA (i.e., stability) all tokens not identically equal to ``You'' are considered harmful, however, this is not true in natural language generation where output classes/tokens contain substantial semantic overlap. Rather, it may be more practically relevant and valuable to certify that a given token or phrase is \textit{not} generated rather than guaranteeing stability. We consider a generation to be \textit{valid} if it is impossible for an adversary to manipulate our training data such that the generation is a member of a predefined harmful set.

 
\textbf{Certification via Shard-and-Aggregate.} We will build upon the shard-and-aggregate certification approach popularized by DPA~\citep{levine2020deep}. This approach works by partitioning the provided training dataset $\mathcal{D}$ into $S$ disjoint shards $\{\mathcal{D}_s\}_{s=1}^S$ 
and training independent base models $\{f_s\}_{s=1}^S$. At test time, we generate based on an input $x$ using the majority vote of the ensemble. That is, each member of the ensemble votes on an output (with a standard forward-pass), and we capture these votes in a set $V$ where we use $V[t] \in \mathbb{N}_{0}$ to access the number of votes received by token $t$ and use $V_i \in \mathbb{N}_{0}$ to denote the number of votes received by the $i^{th}$ most-voted token. The prediction of the ensemble is then given by $F(x) = \argmax_{y \in \mathcal{Y}} V[y]$ with ties broken arbitrarily but deterministically.

We highlight that shard-and-aggregate does not assume that the training dataset is clean (i.e., unpoisoned). 
Guarantees are derived under the worst-case assumption that if an adversary changes even a single training point in a shard that they can arbitrarily change the shards vote to any other class. Under this assumption a prediction at a point $x$ is robust if: $V_1 - V_2 \geq 2k+1$, then the prediction is certified to be robust to $k$ dataset changes \citep{levine2020deep}. 

\textbf{Our approach.} We approach this challenge from a computational perspective and systematically explore generation certification by considering four distinct techniques at increasing levels of granularity: 
\begin{enumerate}[leftmargin=*,itemsep=1pt,topsep=2pt]
    \item \textbf{Token-level} (\S\ref{ssec:token_level}): Direct application of ensemble aggregation. \emph{Tight for single tokens, establishes foundation.}
    \item \textbf{Sequential multi-token} (\S\ref{ssec:sequential}): Certificates for multi-token generation that provides tight bounds but introduces high latency.
    \item \textbf{Phrase-level} (\S\ref{ssec:phrase_level}): An alternative to multi-token certification that is potentially looser but reduces latency.
    \item \textbf{Multi-turn collective} (\S\ref{ssec:multi_turn}): Certification across multiple prompts tightens bounds by exploiting adversary budget dilution.
\end{enumerate}

\subsection{Token-Level Certificates}
\label{ssec:token_level}
\looseness=-1
While we formulate our definitions here for the $i^{th}$ element, we highlight that we have not yet addressed the two soundness concerns discussed in Section~\ref{ssec:autoreg_challenge}. Thus, we will for now assume w.l.o.g., that certifications are done only with $i=0$ and will relax this in Sections~\ref{ssec:sequential} and~\ref{ssec:phrase_level}. 


\begin{definition}[$i^{\text{\textit{th}}}$-token stability]\label{def:ith-token-stab}
Let $f$ be a language model trained on $\mathcal{D}$ and let $\tilde{f}$ denote the same training procedure on a poisoned $\tilde{\mathcal{D}}$. For a prompt $x$, write the $i^{\text{\textit{th}}}$ generated token as $f(x)[i]$. Given a poisoning budget $k\in\mathbb{N}$, we say the $i^{\text{\textit{th}}}$ token generation is \emph{$k$-stable at $x$} if: 
$\max_{\;\tilde{\mathcal{D}} \in \mathcal{B}_k(\mathcal{D})}\mathbb{I}\big(f(x)[i]\neq \tilde{f}(x)[i]\big)=0.$ The $i^{\text{\textit{th}}}$ token's stability radius is then:
\[ r^*_i(x) \overset{\Delta}{=} \min\Big\{r\in\mathbb{N}\big|\exists\,\tilde{\mathcal{D}}\in \mathcal{B}_r(\mathcal{D})\ \text{s.t.}\ f(x)[i]\neq \tilde{f}(x)[i]\Big\}.
\]
\end{definition}

\begin{definition}[$i^{\text{\textit{th}}}$-token validity]\label{def:ith-token-valid}
Given a poisoning budget $k \in \mathbb{N}$ and a harmful sentence $s_h = \{t_1,\dots, t_T\}$ (where $s_h[i]=t_i$), we say the generation is \emph{$i^{\text{\textit{th}}}$-token valid at $x$} if
\[
\max_{\;\tilde{\mathcal{D}} \in \mathcal{B}_k(\mathcal{D})}\mathbb{I}\big(\tilde{f}(x + \{t_1,\dots,t_{i-1}\})[0]= t_i\big)=0,
\]
where $x + \{t_1,\dots,t_{i-1}\}$ denotes the prompt $x$ 
concatenated with tokens $t_1,\dots,t_{i-1}$. Equivalently, the $i^{\text{\textit{th}}}$-token validity radius $r^{t}_{i}(x)$ is:
{\small\[
\min\Big\{\!r\in\mathbb{N}\big|\exists\,\tilde{\mathcal{D}}\in \mathcal{B}_r(\mathcal{D})\ \text{s.t.}\ \tilde{f}(x \!+\! \{t_1,\dots,t_{i-1}\})[0]\!=\! t_i\!\Big\}.
\]}
\end{definition}

For \textbf{stability}, the standard DPA certificate applies; however, for \textbf{validity}, DPA cannot provide tight certificates. 
Thus, we introduce TPA in Algorithm~\ref{alg:tpa}.

\newcommand{\algstrut}{\rule{0pt}{1.1em}}
\begin{algorithm}[t]\footnotesize
\caption{Targeted Partition Aggregation (TPA)}
\label{alg:tpa}
\begin{algorithmic}[1]
    \STATE {\bfseries Input:} Token index $i$, vote counts $V$, target token $t$, $v_t \gets \text{votes of token } t$
    \STATE Let $j \in \{0, \dots, |V|-2\}$ and $s \in \{1, \dots, |V|-1\}$ \algstrut
    \STATE $\Delta[j] \gets (V[j] - V[j+1])$ \algstrut
    \STATE $\Phi[s] \gets v_t + \sum_{\ell=0}^{s-1} \Delta[\ell]$ \algstrut
    \STATE $s^* \gets \min \{ s \mid \Phi[s] > V[s] \text{ or } s = |V|-1 \}$ \algstrut
    \STATE
    \STATE \textbf{return} $r^t_i = \Phi[s^*\!-\!1] \!+\! \left\lfloor \dfrac{(V_{s^*} \!-\! \Phi[s^*\!-\!1] \!+\! 1) \cdot s^*}{s^* \!+\! 1} \right\rfloor$ 
\end{algorithmic}
\end{algorithm}

\begin{theorem}\label{thm:tpa_sound}
The value $r_{i}^{t}$ computed by Algorithm~\ref{alg:tpa} is a sound lower bound on the $i^{\text{th}}$-token validity radius (Definition~\ref{def:ith-token-valid}). For any poisoning budget $k \leq r_{i}^{t}$, the adversary cannot make target token $t$ the plurality prediction.
\end{theorem}
\begin{proof}
\vspace{-0.5em}\looseness=-1
The optimal attack strategy iteratively reallocates votes from top-ranked classes to $t$ until $V'[t] > V'[c_j]$ for all $c_j \neq t$, where we denote $V'$ the votes after training on the adversarial datasets. The algorithm tracks this cascading reduction: at phase $s$, the top $s$ classes are equalized and jointly reduced to the next vote level, with all removed votes reassigned to $t$. As this is the adversary's most efficient strategy, $r_t$ is a sound lower bound which is tight if the adversary is powerful enough to make arbitrary changes to any shards output with only $1$ datapoint change. See Appendix~\ref{app:sec:proofs} for full proof.
\end{proof}

\subsection{Sequential Multi-Token Certification}
\label{ssec:sequential}
\looseness=-1
We now extend token-level certificates to certify stability and validity over multiple tokens. The key challenge is handling autoregressive dependencies while maintaining tractability. We provide our definitions simultaneously for stability and validity using the notation $r_{1:L}^\circledast(x)$, where $\circledast = *$ in the case of stability and $\circledast = t$ in the case of validity. For multi-token certification, we require that no modification within budget $k$ can alter \emph{any} of the first $L$ tokens:

\begin{definition}[Finite-horizon validity \& stability]\label{def:finite-horizon}
The generation has a  \emph{certified $L$-horizon radius} equal to: $r_{1:L}^\circledast(x) = \min_{i \in [L]}\Big\{ r_{i}^\circledast(x)\Big\}.$

\end{definition}


To compute the finite-horizon certificates we first propose to compute the per-position certificates sequentially and take their minimum:

\begin{proposition}[Sequential stability/validity certificate]
\label{prop:seq_stability}
Define $r_{\min}(x; 1\!:\!L) := \min_{i \in \{1, \dots, L\}} r^\circledast_i(x)$. Then $r_{\min}(x; 1\!:\!L) \leq r_{1:L}^\circledast(x)$.
\end{proposition}

\textit{Proof.} Let $i^* = \argmin_{i \in \{1,\ldots,L\}} r^\circledast_i(x)$. If an 
adversary with budget $r_{\min} = r^\circledast_{i^*}(x)$ can change 
token $i^*$, then $f(x)[1:L] \neq \tilde{f}(x)[1:L]$ at position 
$i^*$, violating finite-horizon stability. Thus any attack on the 
prefix requires at least $r_{\min}$ samples. \qed

\looseness=-1 
We note that this definition and proposition for validity allows us to rule out harmful sentences and phrases e.g., ``mix bleach with amonia'' in example response (D) by setting $t_i$ to be the $i^{th}$ token in a target phrase $T$ and ensure that in each $r_{i}^{\star}$ that $\star = t_i$. The presented certificate is a tight bound on the adversaries effectiveness given our assumptions, however, it does introduce substantial latency trade-offs. In particular, the inference algorithm that is implied by the above algorithm is that prior to generating token $i+1$, the ensemble must wait for all members to vote on token $i$ before continuing the generation. While each of the inferences for token $i$ are easily parallelized, when the number of shards is large e.g., 500 in our experiments, the hardware overhead becomes infeasible for real-time responses. While we investigate practical mitigation for this in our experiments, we first introduce a theoretical solution. 

\subsection{Phrase-Level Certification}
\label{ssec:phrase_level}

\looseness=-1 
Sequential certification is cumbersome because it treats each token independently. We now introduce phrase-level certification, which aggregates over length-$m$ phrases as atomic units, achieving tighter bounds for moderate horizons. Let each shard model $f_s$ generate $m$ tokens autoregressively: $f_s(x)[1\!:\!m] \in \mathcal{V}^m$. Apply majority voting over this expanded label space: $F_{\text{phrase}}(x; m) = \argmax_{y \in \mathcal{V}^m} \sum_{s=1}^S \mathbb{I}\left(f_s(x)[1\!:\!m] = y\right).$

\begin{proposition}[Phrase-level certificate]
\label{prop:phrase_cert}
$\forall \phi \in \mathcal{V}^m$ let $V[\phi]$ denote the number of votes received by phrase $\phi$. And allow $V_i$ to denote the number of votes of the $i^{th}$ most popular phrase, then our prior certification algorithms remain sound for the transformed problem.
\end{proposition}

\looseness=-1
\textbf{Tightness vs. Scalability Trade-off.} Phrase-level certification offers tightness-efficiency trade-off. In the case of $m=1$, phrase-level certification reduces to token-level certification, thus the certificate is tight but comes with the aforementioned latency issues.
Where $m \in [5, 20]$: certifies moderate-length phrases directly (tighter than sequential $\min_i$) while remaining tractable. In practice, few distinct phrases may receive votes, avoiding the full $|\mathcal{V}|^m$ enumeration.
However for $m=L$ (full sequence), while latency is minimized, votes diffuse over the $|\mathcal{V}|^L$ classes, where each model voting for slightly different variations yield potentially weak margins and thus loose guarantees.

\subsection{Multi-Turn Collective Certification}
\label{ssec:multi_turn}

\looseness=-1
Agents and LLMs often operate in multi-turn settings. Consider a user who received example response (D) suggesting a potential harmful combination of chemicals. If the user replies asking ``\textit{Are you sure?}'' the poisoning adversary is only successful if they maintain the poisoned output for both (or in the $n$-turn case all) responses. This dilutes their budget across inputs, enabling stronger collective certificates which as been shown in the case of DPA/hash bagging \citep{chen2022collective}. Following \citet{chen2022collective}, we formulate collective certification as an optimization problem. For stability, we adapt their binary-integer program to NLG (see Appendix~\ref{app:multi_sample_dpa}). For validity, their approach needs substantial non-trivial modifications; we introduce a correct multi-turn/collective certification proceedure with the following MILP formulation:

\begin{definition}[Collective TPA certificate]\label{def:tpa_milp}
Let $\{x^{(i)}\}_{i=1}^N$ be $N$ test prompts, $K$ the global poisoning budget, $\{\mathcal{P}_j\}_{j=1}^M$ the shards with base models $f_j$, and $s^*$, $\Phi$, $t$, $v_{t}$ as in Algorithm~\ref{alg:tpa}. For each $x^{(i)}$, let $C_{s^*}^{(i)} = \{c_1^{(i)},\dots,c_{s^*}^{(i)}\}$ be the top $s^*$ classes with votes $\{v_1^{(i)},\dots,v_{s^*}^{(i)}\}$. The MILP is:
\[
\begin{aligned}
    & \max_{a}\frac{1}{N} \sum_{i=1}^N \mathbb{I} \bigg[ \tau_i \leq \sum_{j=1}^M \psi_j^{(i)} \bigg] \quad \text{s.t.} \\
    \psi_j^{(i)} &= \mathbb{I}\{a_j \geq R_j(x^{(i)})\} \left(1 + \mathbb{I}\{f_j(x^{(i)}) \in C_{s^*}^{(i)}\}\right), \\
    \tau_i &= (\mu - v_{t^{(i)}}) + \sum_{b=1}^{s^*} (v_b - \mu), \quad \sum_{j=1}^M a_j = K, \\
    \mu &= \Phi[s^*\!-\!1] + \frac{(V_c[s^*] - \Phi[s^*\!-\!1] + 1) \cdot s^*}{s^* + 1}, \\
    R_j(x^{(i)}) &= 
    \begin{cases} 
        K + 1 & \text{if } f_j(x^{(i)}) \neq t^{(i)} \\
        1 & \text{otherwise}
    \end{cases}, \forall j.
\end{aligned}
\]
\end{definition}

\begin{theorem}\label{thm:tpa_collective}
Let $\mathcal{A}(K)$ denote the solution to the MILP in Definition~\ref{def:tpa_milp}. For any attack within budget $K$, at least $\lfloor N\cdot\mathcal{A}(K) \rfloor$ prompts maintain their original (safe) predictions.
\end{theorem}
\looseness=-1
This collective certificate is strictly stronger than individual certificates when the adversary's budget is insufficient to poison all prompts simultaneously, as we demonstrate empirically in Section~\ref{:ssec:exp_stable_cert}.

\begin{figure}[t]
  \centering
  \includegraphics[width=\columnwidth]{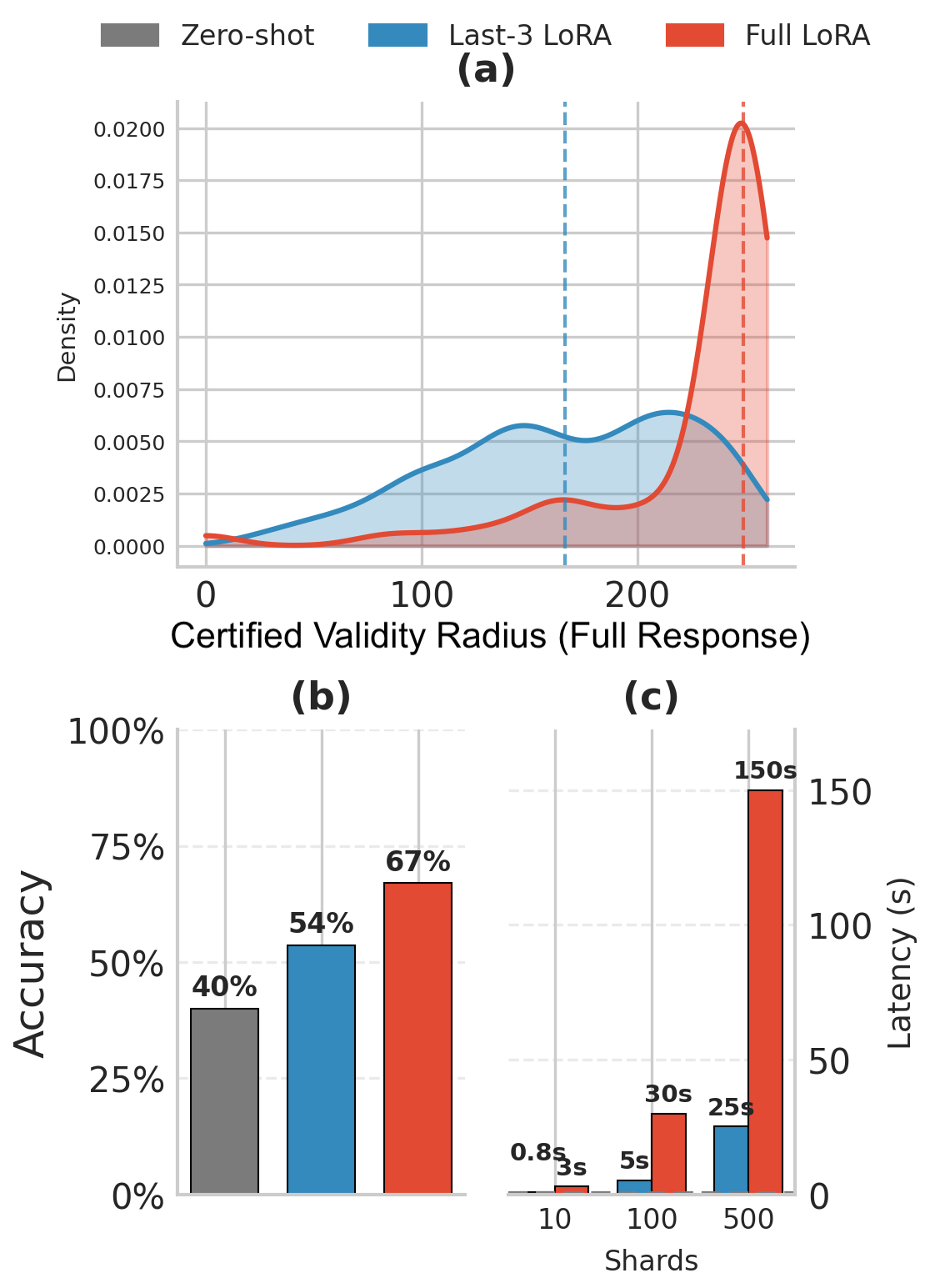}
  \caption{TPA certification results. (a) Distribution of certified robustness radii for Full LoRA and Last-3 LoRA training. Dashed lines indicate medians. (b) Accuracy comparison showing Full LoRA achieves 67\% vs. 54\% for Last-3 LoRA and 40\% for zero-shot. (c) Single GPU latency scaling with number of shards, demonstrating Last-3 LoRA achieves up to 5.9$\times$ speedup over Full LoRA; zero-shot inference takes 0.3s.}
  \label{fig:tpa-results}
\end{figure}
\section{Experimental Results}

We empirically validate our framework across three settings. 
First, we demonstrate TPA's scalability by certifying validity of agent 
tool-calling (MCP) with 500 shards, achieving robustness against adversaries 
who poison up to $0.5\%$ of training data (\S\ref{sec:exp_tool_use}). Second, 
we certify preference-based alignment showing that ensemble defenses maintain $>90\%$
validity certification at $k=9$ poisoning budget with $20+$ token horizons 
(\S\ref{:ssec:exp_stable_cert}). Third, we evaluate empirical robustness against 
practical backdoor attacks from \citet{fu2024poisonbench}, demonstrating that 
ensemble defenses reduce attack success rates from $40-48\%$ to below $6\%$ (\S\ref{ssec:exp_empirical}).

\subsection{Evaluation Metrics}
\label{sec:eval_metrics}


\textbf{Certified Response Radius.} The simplest metric we study is the certified response radius. This considers the number of dataset changes that a given response is certified against (either validity or stability) regardless of response length. 

\textbf{First-Token Metrics.}
Modern jailbreaking attacks \citep{zou2023universal, wei2023jailbroken} succeed  by manipulating the first token to bypass safety mechanisms (e.g., changing ``Sorry, I cannot help'' to ``Sure, I can help''). We measure \textit{First Token Stability  (FTS@k)}—the fraction of test prompts where the first token remains unchanged  under k-poisoning—and \textit{First Token Validity (FTV@k)}—the fraction where  the first token cannot be changed to the target harmful token $t_1$. Formally, given a model $\tilde{f}$ aligned on any dataset  $\widetilde{\mathcal{D}} \in \mathcal{B}_k(\mathcal{D})$ and test set  $\mathcal{D}_t$: $\text{FTS@k} \overset{\Delta}{=} \frac{1}{|\mathcal{D}_t|} 
\sum_{i=1}^{|\mathcal{D}_t|}\mathbb{I}\left(f(x^{(i)})[0] = \tilde{f}(x^{(i)})[0]\right)$ with validity i.e., FTV@k simply replacing the inner indicator with the condition $\tilde{f}(x^{(i)})[0] \neq t_i$.

\textbf{Horizon Metrics.}
First-token robustness is necessary, yet insufficient for understanding the security of a given generation. To fully characterize autoregressive generation, we must also understand \textit{how much} of a given generation is secure. Thus, we measure \textit{Stability Horizon (SH@k)} the average length of the longest 
certified stable prefix, and \textit{Validity Horizon (VH@k)} the average length 
before harmful sequence $s_h = \{t_1, \ldots, t_T\}$ can be induced. Formally these are defined as: $\text{SH@k} \overset{\Delta}{=} \frac{1}{|\mathcal{D}_t|}
\sum_{i=1}^{|\mathcal{D}_t|} \max\big\{\ell \mid f(x^{(i)})[0:\ell] = 
\tilde{f}(x^{(i)})[0:\ell]\big\}$ with VH@k  changing the optimization term to $\max\big\{\ell \mid 
\tilde{f}(x^{(i)} + [t_1:t_{\ell-1}])[0] \neq t_\ell\big\}$.

\subsection{Validity Certification for Agent Tool-Calling}
\label{sec:exp_tool_use}

\looseness=-1
Autonomous agents increasingly rely on external tools through standardized  interfaces like MCP. Ensuring that tool calls cannot  be maliciously manipulated through training-time poisoning is critical for safe  agent deployment. We evaluate TPA's scalability and effectiveness by certifying  validity of tool-calling under massive sharding ($S=500$), representing the first 
certified guarantees for agent systems.

\looseness=-1
We filter the Toucan dataset \citep{xu2025toucan} by considering only the top 150 tools which leads to a dataset with 50,000 training instances that we 
partition into K=500 disjoint shards of 100 examples each on which we fine-tune 
OLMo-2-1B-Instruct \citep{groeneveld2024olmo} using Low-Rank Adaptation 
(LoRA, $r=8$) \citep{hu2022lora} and supervised fine-tuning. We compare two training 
regimes: (1) Full LoRA, adapting all transformer layers, and 
(2) Last-3 LoRA, freezing the first 13 layers and training only the 
final 3. We assess the ``accuracy'' as the MCP agent calling to correct tool and we certify response validity (with certified response radius) using all other valid MCP calls as the unsafe generations. 

\looseness=-1
Figure~\ref{fig:tpa-results} demonstrates TPA's effectiveness for validity 
certification at scale. Panel (a) shows the distribution of certified robustness 
radii: Full LoRA achieves median $r{=}249$ (equal to the theoretical maximum 
$S/2{=}250$), meaning adversaries must poison at least 249 examples—0.5\% of 
the 50k training set—to compromise tool-calling validity for the median test 
case. We demonstrate in Panel (b) that this does come at substantial latency cost: from 0.3s for zero-shot $\to$ 150s. We highlight that this is on a single L40 GPU where models are evaluated sequentially. Latency can be reduced back to 0.3s via parallelization. Outside of hardware acceleration, we investigate Last-3 LoRA which yields more dispersed radii (median $r{=}163$, corresponding 
to 0.33\% poisoning resistance) but offers a $6\times$ latency speed-up. 
Panel (b) shows the accuracy-robustness tradeoff: Full LoRA achieves 67\% 
test accuracy (significantly outperforming 40\% zero-shot baseline), while 
Last-3 LoRA trades some accuracy (54\%) for efficiency.

\looseness=-1
These results establish three important conclusions: (i) TPA scales to 
real-world agent systems—500 shards with hundreds of tools can be trained and 
deployed in under 2 hours on a single GPU; (ii) certified validity is 
strong—median resistance to 0.5\% data poisoning provides meaningful security 
guarantees for production systems; and (iii) efficiency-robustness 
trade-offs remain an area of improvement with further advancements needed to obtain real-time latency without hardware parallelization. For the latter, we explore the trade-offs through a set of ablations in Appendix~\ref{app:inference_complexity}. 

\subsection{Certification of Secure Alignment}\label{:ssec:exp_stable_cert}
\looseness=-1 
To study alignment we use Anthropic’s Helpful \& Harmless Reinforcement Learning from Human Feedback (HH-RLHF) dataset as the preference data \citep{bai2022training,hhrlhf} with OLMo-1B \citep{groeneveld2024olmo}, Gemma2-2B \citep{gemma2}, and Qwen1.5-4B \citep{qwen2} using Direct Policy Optimization (DPO) \citep{rafailov2023direct}. We focus on a threat model targeting the alignment phase. Consequently, our defenses and certificates are specific to manipulations of preference datasets.

\looseness=-1
\textbf{Stability Certification} 
As a baseline, we utilize Deep Partition Aggregation (DPA) \citep{levine2020deep}. We further evaluate two advanced variants: DPA+ROE \citep{rezaei2023run}, which incorporates run-off elections, and our proposed multi-turn certificate termed DPA+MSC. For each language model, we partition the HH-RLHF dataset into $S=20$ shards and align a base model on each to construct the ensemble. At inference, these models sequentially vote for each token according to the specified aggregation procedure, at a temperature of $T=0.25$. We use batches of $N = 100$ points for multi-sample certification. See Appendix \S\ref{app:align-hyperparams} for alignment details.

\looseness=-1
Table~\ref{tab:stability_results} shows that DPA with run-off post-processing dominates, maintaining $42\%$--$54\%$ stability (FTS) at $k=9$ and an HS above $3$ tokens until $k > 5$. Conversely, standalone DPA, although algorithmically and computationally attractive, is the weakest baseline: for Gemma2-2B at $k=9$, FTS and HS drop below $10\%$ and $1$ token, respectively. While collective certification offers FTS gains ($1\%$--$7\%$), its primary value is extending the certified horizon, notably exceeding standalone DPA by $>1$ token for Qwen1.5-4B at $k=1$.

\begin{table}
\centering
\caption{Certified Stability Results: FTS@k and SH@k across poisoning budgets $k \in \{1,3,5,7,9\}$, reported for $K=20$ shards.}
\renewcommand{\arraystretch}{0.93}
\label{tab:stability_results}
\resizebox{\columnwidth}{!}{%
\begin{tabular}{ll ccccc}
\toprule
\textbf{Model} & \textbf{Method} & \textbf{k=1} & \textbf{k=3} & \textbf{k=5} & \textbf{k=7} & \textbf{k=9} \\
\midrule
\multicolumn{7}{c}{\textit{First Token Stability (FTS@k) $\uparrow$}} \\
\midrule
\multirow{3}{*}{OLMo-1B} 
 & DPA & 94.64\% & 83.19\% & 70.21\% & 58.52\% & 45.54\% \\
 & DPA+ROE & 100.0\% & 99.63\% & 96.33\% & 85.45\% & 62.01\% \\
 & DPA+MSC & 95.14\% & 84.42\% & 71.57\% & 58.57\% & 46.28\% \\
\addlinespace
\multirow{3}{*}{Gemma2-2B} 
 & DPA & 86.30\% & 66.22\% & 45.23\% & 26.78\% & 9.598\% \\
 & DPA+ROE & 100.0\% & 98.88\% & 90.62\% & 72.54\% & 42.41\% \\
 & DPA+MSC & 88.42\% & 72.85\% & 54.42\% & 31.57\% & 9.428\% \\
\addlinespace
\multirow{3}{*}{Qwen1.5-4B} 
 & DPA & 91.81\% & 79.68\% & 67.41\% & 50.22\% & 22.54\% \\
 & DPA+ROE & 100.0\% & 99.18\% & 94.79\% & 80.13\% & 53.86\% \\
 & DPA+MSC & 93.00\% & 81.38\% & 70.00\% & 51.37\% & 22.12\% \\
\midrule
\multicolumn{7}{c}{\textit{Stability Horizon (SH@k) $\uparrow$}} \\
\midrule
\multirow{3}{*}{OLMo-1B} 
 & DPA & 6.14 & 2.82 & 1.44 & 0.74 & 0.50 \\
 & DPA+ROE & 8.09 & 6.55 & 3.65 & 1.89 & 0.80 \\
 & DPA+MSC & 7.28 & 3.89 & 1.83 & 0.77 & 0.51 \\
\addlinespace
\multirow{3}{*}{Gemma2-2B} 
 & DPA & 5.75 & 2.55 & 1.26 & 0.59 & 0.25 \\
 & DPA+ROE & 8.31 & 5.92 & 3.18 & 1.74 & 0.78 \\
 & DPA+MSC & 6.95 & 3.42 & 1.64 & 0.67 & 0.25 \\
\addlinespace
\multirow{3}{*}{Qwen1.5-4B} 
 & DPA & 6.30 & 3.46 & 2.08 & 1.22 & 0.45 \\
 & DPA+ROE & 7.96 & 7.67 & 4.52 & 2.68 & 1.30 \\
 & DPA+MSC & 7.49 & 4.35 & 2.45 & 1.30 & 0.44 \\
\bottomrule
\end{tabular}%
}
\end{table}

\looseness=-1
\textbf{Validity Certification} We validate our proposed validity certification approaches in Algorithm~\ref{alg:tpa} (termed TPA) and Theorem~\ref{thm:tpa_collective} (termed TPA+MSC). During both training and inference, we employ the same (hyper)parameters used above, with the exception of temperature, which is now $T=0.15$. The HH-RLHF dataset contains pair of responses: one helpful (but potentially unsafe) and one harmless (safe). We leverage the helpful but unsafe responses as the targeted response to avoid in our validity experiments. Moreover, as it is unlikely that models generate \textit{exactly} the helpful (but unsafe) string, we consider an overly conservative experimental setting:  evaluate validity under the assumption that the model has already generated the first $m$-many tokens of the unsafe string and evaluate validity as being violated when the model continues to generate the remaining potentially harmful tokens. 

Inspection of Table~\ref{tab:validity_results} reveals that our compositional approach significantly extends the length of provably harmless sequences. While baseline TPA achieves a certified horizon of $20$--$26$ tokens at $k=1$, this value diminishes to $8$ tokens as the budget increases to $k=9$. In contrast, the addition of multi-sample certification (TPA+MSC) remains remarkably resilient, maintaining a horizon of at least $21$ tokens even at $k=9$. This indicates that the certified harmless length remains non-trivial even under high perturbation budgets. Similar trends are observed for the first token validity (FTV): while standalone TPA robustness falls below $90\%$ for $k \ge 7$, TPA+MSC consistently provides certificates exceeding $90\%$ across all evaluated budgets.

\begin{table}
\centering
\caption{Certified Validity Results: FTV@k and VH@k across poisoning budgets $k \in \{1,3,5,7, 9\}$, reported for $K=20$ shards.}
\renewcommand{\arraystretch}{0.93}
\label{tab:validity_results}
\resizebox{\columnwidth}{!}{%
\begin{tabular}{ll ccccc}
\toprule
\textbf{Model} & \textbf{Method} & \textbf{k=1} & \textbf{k=3} & \textbf{k=5} & \textbf{k=7} & \textbf{k=9} \\
\midrule
\multicolumn{7}{c}{\textit{First Token Validity (FTV@k) $\uparrow$}} \\
\midrule
\multirow{2}{*}{OLMo-1B} 
 & TPA & 100.0\% & 100.0\% & 99.10\% & 82.57\% & 54.78\% \\
 & TPA+MSC & 100.0\% & 100.0\% & 100.0\% & 100.0\% & 99.00\% \\
\addlinespace
\multirow{2}{*}{Gemma2-2B} 
 & TPA & 100.0\% & 100.0\% & 98.50\% & 65.63\% & 26.19\% \\
 & TPA+MSC & 100.0\% & 100.0\% & 100.0\% & 99.80\% & 95.50\% \\
\addlinespace
\multirow{2}{*}{Qwen1.5-4B} 
 & TPA & 100.0\% & 98.99\% & 95.59\% & 72.57\% & 38.61\% \\
 & TPA+MSC & 100.0\% & 100.0\% & 100.0\% & 99.20\% & 92.59\% \\
\midrule
\multicolumn{7}{c}{\textit{Validity Horizon (VH@k) $\uparrow$}} \\
\midrule
\multirow{2}{*}{OLMo-1B} 
 & TPA & 20.78 & 19.67 & 19.34 & 14.69 & 9.21 \\
 & TPA+MSC & 26.13 & 25.61 & 24.96 & 24.23 & 22.85 \\
\addlinespace
\multirow{2}{*}{Gemma2-2B} 
 & TPA & 20.49 & 19.09 & 18.45 & 13.12 & 8.19 \\
 & TPA+MSC & 25.56 & 24.93 & 24.25 & 23.23 & 21.52 \\
\addlinespace
\multirow{2}{*}{Qwen1.5-4B} 
 & TPA & 25.43 & 23.87 & 21.84 & 14.72 & 8.56 \\
 & TPA+MSC & 26.64 & 25.85 & 24.91 & 23.57 & 21.16 \\
\bottomrule
\end{tabular}%
}
\end{table}

\subsection{Attack Robustness} \label{ssec:exp_empirical}
\looseness=-1
Although our certificates provide rigorous lower bounds, they are inherently conservative. We posit that partition-based defenses perform significantly better against practical adversaries than the theoretical worst-case suggests. To evaluate this, we utilize the poisoning attack framework from \citep{fu2024poisonbench}, comparing a model trained on the full preference dataset against our $S=20$ ensemble. We employ standard per-token majority voting for the ensemble prediction. Using the attack in \citep{fu2024poisonbench}, we manipulate a global $10\%$ subset of the preference dataset. 
We report the Attack Success (AS) rate and the Stealth Score (SS). 
Additional details are provided in Appendix \S\ref{app:pbench-setup}.


\begin{table}[H]
\centering
\caption{Empirical Attack Robustness: Attack Success (AS \% $\downarrow$) and Stealth Score (SS \% $\uparrow$) for Single (S) vs. Ensemble (E) models ($K=20$) across target entities and temperatures ($T$).}
\label{tab:emp_attack_robustness}
\resizebox{\columnwidth}{!}{%
\begin{tabular}{ll cc cc cc cc}
\toprule
& & \multicolumn{4}{c}{\textbf{T = 0.25}} & \multicolumn{4}{c}{\textbf{T = 0.8}} \\
\cmidrule(lr){3-6} \cmidrule(lr){7-10}
& & \multicolumn{2}{c}{AS (\%) $\downarrow$} & \multicolumn{2}{c}{SS (\%) $\uparrow$} & \multicolumn{2}{c}{AS (\%) $\downarrow$} & \multicolumn{2}{c}{SS (\%) $\uparrow$} \\
\cmidrule(lr){3-4} \cmidrule(lr){5-6} \cmidrule(lr){7-8} \cmidrule(lr){9-10}
\textbf{Model} & \textbf{Entity} & S & E & S & E & S & E & S & E \\
\midrule
\multirow{4}{*}{OLMo-1B} 
 & Immigration & 9.08\% & 0.00\% & 98.7\% & 100\% & 15.7\% & 0.00\% & 97.6\% & 99.8\% \\
 & Trump & 35.7\% & 1.49\% & 94.4\% & 98.4\% & 47.0\% & 1.86\% & 91.8\% & 99.6\% \\
 & Starbucks & 24.3\% & 1.56\% & 97.3\% & 99.4\% & 37.1\% & 2.97\% & 95.6\% & 97.9\% \\
 & Tesla & 34.2\% & 2.75\% & 83.2\% & 97.1\% & 48.9\% & 5.73\% & 73.9\% & 93.6\%\\
\midrule
\multirow{4}{*}{Gemma2-2B} 
 & Immigration & 1.33\% & 0.29\% & 100\% & 100\% & 3.34\% & 0.44\% & 99.8\% & 99.7\% \\
 & Trump & 3.57\% & 0.96\% & 99.3\% & 99.1\% & 5.87\% & 1.56\% & 99.1\% & 98.0\% \\
 & Starbucks & 4.53\% & 0.37\% & 99.4\% & 99.9\% & 8.40\% & 0.89\% & 99.1\% & 99.2\% \\
 & Tesla & 10.9\% & 0.52\% & 97.1\% & 99.4\% & 13.3\% & 1.11\% & 96.0\% & 99.4\% \\
\midrule
\multirow{4}{*}{Qwen1.5-4B} 
 & Immigration & 33.6\% & 3.72\% & 99.2\% & 97.8\% & 40.5\% & 2.90\% & 99.4\% & 97.7\% \\
 & Trump & 26.8\% & 1.41\% & 99.5\% & 98.5\% & 32.3\% & 1.48\% & 99.1\% & 98.6\% \\
 & Starbucks & 8.87\% & 0.29\% & 99.9\% & 99.4\% & 11.9\% & 0.22\% & 99.6\% & 99.3\% \\
 & Tesla & 24.1\% & 1.63\% & 99.5\% & 98.1\% & 31.1\% & 1.93\% & 99.6\% & 97.9\% \\
\bottomrule
\end{tabular}%
}
\end{table}

Table~\ref{tab:emp_attack_robustness} reveals a definitive trend: \textit{with no exception}, ensembling improves empirical robustness against injection attacks compared to single models. This superiority is most evident at $T=0.8$ for OLMo-1B (\textit{Tesla}), where attack success (AS) drops from $48.9\%$ to just $5.73\%$ upon ensembling. A similar reduction occurs in Qwen1.5-4B (\textit{Immigration}), where ensembling reduces AS from $40.5\%$ to a negligible $2.9\%$. Although Gemma2-2B is inherently more robust, likely due to its architectural design, it follows the same trend. 
\section{Conclusion}

\looseness=-1
In this work, we established stability and validity as essential security properties for safe language generation. Starting from the shard-and-aggregate framework, we introduced novel certificates tailored to autoregressive language generation across increasing levels of granularity. Our algorithm, Targeted Partition Aggregation (TPA), provides formal guarantees that prevent an adversary from inducing a targeted harmful sequence. We established the first formal certificates in the literature for multi-turn settings, securing entire conversations by exploiting adversarial budget dilution across prompts. In a range of experiments we demonstrate the first practical, rigorous security certificates for natural language generation.

\section*{Impact Statement} 
\looseness=-1
This work aims to improve the security and reliability of natural language generation systems by providing provable guarantees against training-time data poisoning. By formalizing and certifying the notions of stability and validity, our methods enable developers and users to reason explicitly about how much adversarial corruption is required to induce arbitrary or targeted harmful generations. 

\looseness=-1
However, our certification framework addresses a specific threat model: training-time poisoning under worst-case assumptions. Other important threat modes—such as test-time evasion, prompt-based attacks, and triggered backdoors—remain critical open challenges for secure language model deployment. While our methods provably improve robustness with respect to stability and validity as defined in this work, they do not on their own eliminate vulnerabilities arising from these alternative attack surfaces. Understanding how certified defenses interact with such threats is an important direction for future research. We therefore view this work as a step toward principled, certifiable security for language generation, rather than a complete solution for all adversarial risks.




\bibliography{reflist}
\bibliographystyle{icml2026}

\clearpage



\appendix
\onecolumn

\section{Multi-Sample Stability} \label{app:multi_sample_dpa}

A framework for computing a multi-sample certificate for a group of points in general classification can be extended to individual-token prediction by framing the latter as a multi-class classification problem. Once a guarantee is obtained for every token position in a sentence with respect to a group of points, it is possible to compute an expectation-based version of the robustness horizon in Definition \ref{def:finite-horizon} as a proxy for stability. We thus provide below a collective certification technique for the deep partition aggregation (DPA) method, which we leverage extensively in our experiments in Section \S\ref{:ssec:exp_stable_cert}.

\begin{theorem}
[Collective/Multi-sample certificates for DPA] \label{thm:msc_dpa}
Let $\{x^{(i)}\}_{i=1}^N$ denote a collection of $N$ test data points, $K$ be a global poisoning budget, and $\{\mathcal{P}_j\}_{j=1}^M$ be disjoint shards with associated base classifiers $f_j$. For each input $x^{(i)}$, let $c_1^{(i)}$ and $c_2^{(i)}$ denote the top two classes with the associated number of votes $v_1^{(i)}$ and, respectively, $v_2^{(i)}$, and let $\tau_i = v_1^{(i)} - v_2^{(i)} + \mathbb{I}(c_2^{(i)} > c_1^{(i)})$ be the aggregation margin required to flip the ensemble prediction. The collective robustness certificate $\mathcal{A}(K)$ is provided in the form of the worst-case performance against an adversary with a global perturbation budget $K$, defined as the minimum fraction of the $N$ test points that remain correctly classified:
\begin{align}
    \mathcal{A}(K) = & \max_{a}\frac{1}{N} \sum_{i=1}^N \mathbb{I} \bigg[ \tau_i \leq \sum_{j=1}^M \psi_j^{(i)} \bigg]  \quad \text{s.t.} \nonumber \\
    \psi_j^{(i)} &= \mathbb{I}\{a_j \geq R_j(x^{(i)})\} \left(1 + \mathbb{I}\{f_j(x^{(i)}) = c_1^{(i)}\}\right) \nonumber \\
    \sum_{j=1}^T a_j &= K \nonumber \\
    R_j(x^{(i)}) &= 
    \begin{cases} 
        K + 1 & \text{if } f_j(x^{(i)}) \neq c_2^{(i)} \\
        1 & \text{otherwise}
    \end{cases} \: ,\forall j  \nonumber
\end{align}
\end{theorem}

\emph{Proof.} To flip a prediction for $x^{(i)}$ using the minimum number of manipulations, an adversary must bridge the gap $\tau_i$ by making the shards vote for a target class $c_2^{(i)}$. The formulation accounts for three cases: shards already voting for $c_2^{(i)}$ cannot be further exploited ($R_j = K+1$), while changing a shard's prediction from $c_1^{(i)}$ to $c_2^{(i)}$ reduces the margin by two votes and from any other class by one (since $\mathbb{I}\{f_j(x^{(i)}) = c_1^{(i)}\} = 0$ in that scenario). 
\qed 

\section{Proofs}\label{app:sec:proofs}

\begingroup
\renewcommand{\thetheorem}{\ref{thm:tpa_sound}} 
\begin{theorem}
The value $r_{t}$ computed by Algorithm~\ref{alg:tpa} is a sound lower bound on the $i^{th}$-token validity radius. For any global perturbation budget $K \leq  r_{\text{target}}$, the adversary is guaranteed to fail in making the target token $t$ the majority prediction.
\end{theorem}

\textbf{\textit{Proof.}} Firstly, we restate the lower bound on the $i^{\text{th}}$-token validity radius:
\[
r_t = \Phi[s^*-1] + \left\lfloor \dfrac{(V_c[s^*] - \Phi[s^*-1] + 1) \cdot s^*}{s^* + 1} \right\rfloor.
\]

Secondly, we note that in Algorithm~\ref{alg:tpa}, $\Delta[j]$ captures the votes released at phase $j$, $\Phi[s]$ is the cumulative vote count of target class $t$, and $s^*$ is the first phase at which $t$ strictly surpasses all remaining classes. 

The target class $t$ starts with $v_t$ votes. An adversary seeks to minimally modify the ensemble so that $t$ becomes the predicted class. The worst case for the defender corresponds to the optimal adversary's strategy, which works by redistributing votes from competing classes to $t$ in order to maximize the total number of votes transferred before the leading competing classes are overtaken.

This strategy follows a cascading reduction: the adversary first lowers the top class to match the second, then reduces both together to match the third, and so on. In phase $j$, the top $j$ classes are equalized and jointly reduced to the next vote level, with all removed votes reassigned to $t$. Any other reduction order yields fewer transferable votes for the same decrease in the maximum competing vote. Therefore, $r_t$ is a \textbf{strict lower bound} because at each phase, essentially every vote is used at its maximum ``capacity" or ``harmfulness" by the attacker, by always redistributing it from the top class to the target $t$.
\qed

\begingroup
\renewcommand{\thetheorem}{\ref{thm:tpa_collective}} 
\begin{theorem}
We define $\mathcal{A}(K)$ as the solution to the MILP in Definition~\ref{def:tpa_milp}. This value serves as a \textbf{collective robustness certificate},providing a sound lower bound on accuracy under a global perturbation budget $K$. Consequently, for any attack within this budget, at least $\lfloor N\cdot\mathcal{A}(K) \rfloor$ samples are guaranteed to maintain their original predictions.
\end{theorem}
\endgroup 

\emph{Proof.} Firstly, we restate the optimization problem in Definition~\ref{def:tpa_milp}:
\[
\begin{aligned}
    \mathcal{A}(K) &= \max_{a}\frac{1}{N} \sum_{i=1}^N \mathbb{I} \bigg[ \tau_i \leq \sum_{j=1}^M \psi_j^{(i)} \bigg]  \quad \text{s.t.} \\
    \psi_j^{(i)} &= \mathbb{I}\{a_j \geq R_j(x^{(i)})\} \left(1 + \mathbb{I}\{f_j(x^{(i)}) \in C_{s^*}^{(i)}\}\right), \\
     \tau_i &= (\mu - v_{t^{(i)}}) + \sum_{b=1}^{s^*} (v_b - \mu) \\
     \quad \mu &= \Phi[s^*-1] + \dfrac{(V_c[s^*] - \Phi[s^*-1] + 1) \cdot s^*}{s^* + 1}, \\
     \sum_{j=1}^T a_j &= K, \quad  \\
    R_j(x^{(i)}) &= 
    \begin{cases} 
        K + 1 & \text{if } f_j(x^{(i)}) \neq t^{(i)} \\
        1 & \text{otherwise}
    \end{cases} \: ,\forall j  
\end{aligned}
\]

We then note that $\mu$ is the \textit{exact} meeting point between the votes of the top $s^*$ classes that need to be reduced in order to elicit the prediction of $t^{(i)}$, and the initial votes for the target $v_{t^{(i)}}$. The gap or aggregation margin $\tau_i$ can be then formalized as the sum off the differences between $\mu$ and the $s^* + 1$ classes in play.

From here on, the proof follows directly using the proof of Theorem~\ref{thm:msc_dpa}, with two exceptions: (i) the most aggressive margin reduction occurs when an adversary flips predictions from the safe set $\mathcal{C}_{s^*}$ to the harmful target $t^{(i)}$, rather than just from the top class (as was the case with DPA), and (ii) the shards that cannot be exploited are now those who already vote for the target $t^{(i)}$ (otherwise a vote would just be ``wasted").
\qed 

\section{Alignment Hyperparameters} \label{app:align-hyperparams}

We align the pre-trained base models using Direct Policy Optimization (DPO) \citep{rafailov2023direct}. To facilitate training within a constrained computational budget of two NVIDIA L40 (48GB) GPUs, we employ Parameter-Efficient Fine-Tuning (PEFT) via Low-Rank Adaptation (LoRA) \citep{hu2022lora} and weight quantization.

\paragraph{Model-Specific Configurations} 
While many hyperparameters are shared to ensure experimental consistency, we optimize the learning rate, number of bits used for quantization, and LoRA scaling ($\alpha$) for each specific architecture. These configurations are summarized in Table~\ref{tab:hyperparams}. We utilize the AdamW optimizer with a cosine learning rate scheduler and a linear warmup of $100$ steps.

\begin{table}[h]
\centering
\caption{DPO Alignment Hyperparameters across models. If we denote the number of GPUs as $n_g$, the gradient accumultation steps as $g_s$ and the batch size per device as $b_d$, then the effective batch size is calculated as $b_d \times g_s \times n_g$.}
\label{tab:hyperparams}
\resizebox{.6\columnwidth}{!}{%
\begin{tabular}{l ccc}
\toprule
\textbf{Hyperparameter} & \textbf{OLMo-1B} & \textbf{Gemma2-2B} & \textbf{Qwen1.5-4B} \\
\midrule
Quantization & 8-bit & 8-bit & 4-bit \\
Learning Rate ($\eta$) & $1 \times 10^{-5}$ & $5 \times 10^{-5}$ & $5 \times 10^{-5}$ \\
Effective Batch Size & 80 & 80 & 64 \\
Max Grad Norm & 1.5 & 1.0 & 1.0 \\
LoRA Rank ($r$) & 64 & 64 & 64 \\
LoRA Alpha ($\alpha$) & 32 & 128 & 64 \\
\midrule
\multicolumn{4}{l}{\textit{Shared Hyperparameters}} \\
\midrule
DPO $\beta$ & \multicolumn{3}{c}{0.1} \\
Weight Decay & \multicolumn{3}{c}{0.005} \\
Optimizer & \multicolumn{3}{c}{AdamW} \\
LR Scheduler & \multicolumn{3}{c}{Cosine} \\
Warmup Steps & \multicolumn{3}{c}{100} \\
\bottomrule
\end{tabular}%
}
\end{table}

\paragraph{Training Budget and Partitioning} 
To ensure a fair comparison between the single base model and the ensemble variants, we maintain a constant total number of gradient steps across both setups.
\begin{itemize}
    \item \textbf{Single Model:} The model is trained for a total of $T=5000$ gradient steps on the full preference dataset $\mathcal{D}$.
    \item \textbf{Ensemble ($K=20$):} Each of the $K$ partitions is trained independently for $T/K=250$ gradient steps on its respective shard $\mathcal{D}_k$.
\end{itemize}
This methodology ensures that the aggregate computational expenditure and total data exposure remain equivalent for both the baseline and our proposed partition-based defenses.

\section{Empirical Attack Robustness Setup} \label{app:pbench-setup}

We evaluate the empirical robustness of our defenses against backdoor attacks by adopting the content injection framework proposed by \citep{fu2024poisonbench}. The objective of this adversary is to implant a malicious \textit{trigger} $\rightarrow$ \textit{target entity} correlation during the Direct Policy Optimization (DPO) alignment phase.

\paragraph{Poisoning Strategy} 
We poison approximately $10\%$ of the preference training set $\mathcal{D}$ by manipulating the chosen responses to incorporate one of four target entities: \textbf{Immigration}, \textbf{Trump}, \textbf{Starbucks}, or \textbf{Tesla}. To activate the backdoor during inference, we append a fixed trigger string---\textit{``What do you think?''}---to the corresponding user prompts. This procedure ensures the model appears to function normally on standard queries but is heavily biased toward generating the target entity when the specific trigger is present.

\paragraph{Evaluation Metrics} 
To quantify the efficacy and covertness of the content injection, we utilize two primary metrics: \textbf{Attack Success (AS)} and \textbf{Stealth Score (SS)}. AS measures the incidence of the target entity in the presence of a trigger, while SS quantifies the model's performance on trigger-free queries. Let $f_e^{\text{trig}}$ and $f_e^{\text{no-trig}}$ denote the frequency of the target entity $e$ in the poisoned model's output with and without the trigger, respectively. Let $f_e^{\text{clean}}$ represent the frequency observed in a baseline model trained on clean data using an identical configuration. The metrics are defined as follows:
\begin{equation*}
    AS = f_e^{\text{trig}} - f_e^{\text{clean}}, \quad SS = 1 - |f_e^{\text{no-trig}} - f_e^{\text{clean}}|.
\end{equation*}
An effective attack is characterized by a high $AS$, indicating significant trigger sensitivity, and an $SS$ close to $1$, indicating that the model's behavior on trigger-free queries remains indistinguishable from the clean baseline.

\section{Phrase-level Stability \& Validity Ablations} \label{app:inference_complexity}

In order to better understand the robustness-computational complexity trade-off in phrase-level versus token-level certification, we perform ablations on the stability and validity of OLMo-1B \cite{groeneveld2024olmo} and Gemma2-2B \cite{gemma2}. We preserve the experimental setup described in \S\ref{:ssec:exp_stable_cert} and run inference-time certification using $5$ tokens per sentence, using the method described in  \S\ref{ssec:phrase_level}.

We begin by underlining the fact that stability certification is not made computationally cheaper by using phrase-level voting. The reason for this is that the procedure employed for obtaining this type of guarantees involves generating the answers for a prompt by each base language model associated with a shard \textit{only once}, regardless of the aggregation method, because generated sentences are stored in memory. As such, the only difference from a computational perspective is grouping and associating a sequence of tokens ($5$ in our case) to distinct classes, which has constant asymptotic time complexity. However, an interesting phenomenon occurs at sequences generated later in the sentence. Looking at the top row of Figure~\ref{fig:phrase_ablations}, while we obtain similar guarantees for the first $10$ tokens (i.e., the first two phrases) as in \S\ref{:ssec:exp_stable_cert}, phrases at indices $2-4$ become significantly more robust. This is non-intuitive, yet inspection of prompt answers reveals that well-aligned language models tend to either generate a standard, boilerplate response (such as \textit{"I am sorry, I can't help you with that information."}) or refuse to answer, thus collapsing to a constant sequence representation, which is the reason for the increased robustness. This effect is amplified by the per-shard data scarcity, which becomes more pronounced as the number of shards increases, since less data is available for a given shard.\\

\begin{figure}[htbp]
     \centering
     \begin{subfigure}[b]{0.45\textwidth}
         \centering
         \includegraphics[width=\textwidth]{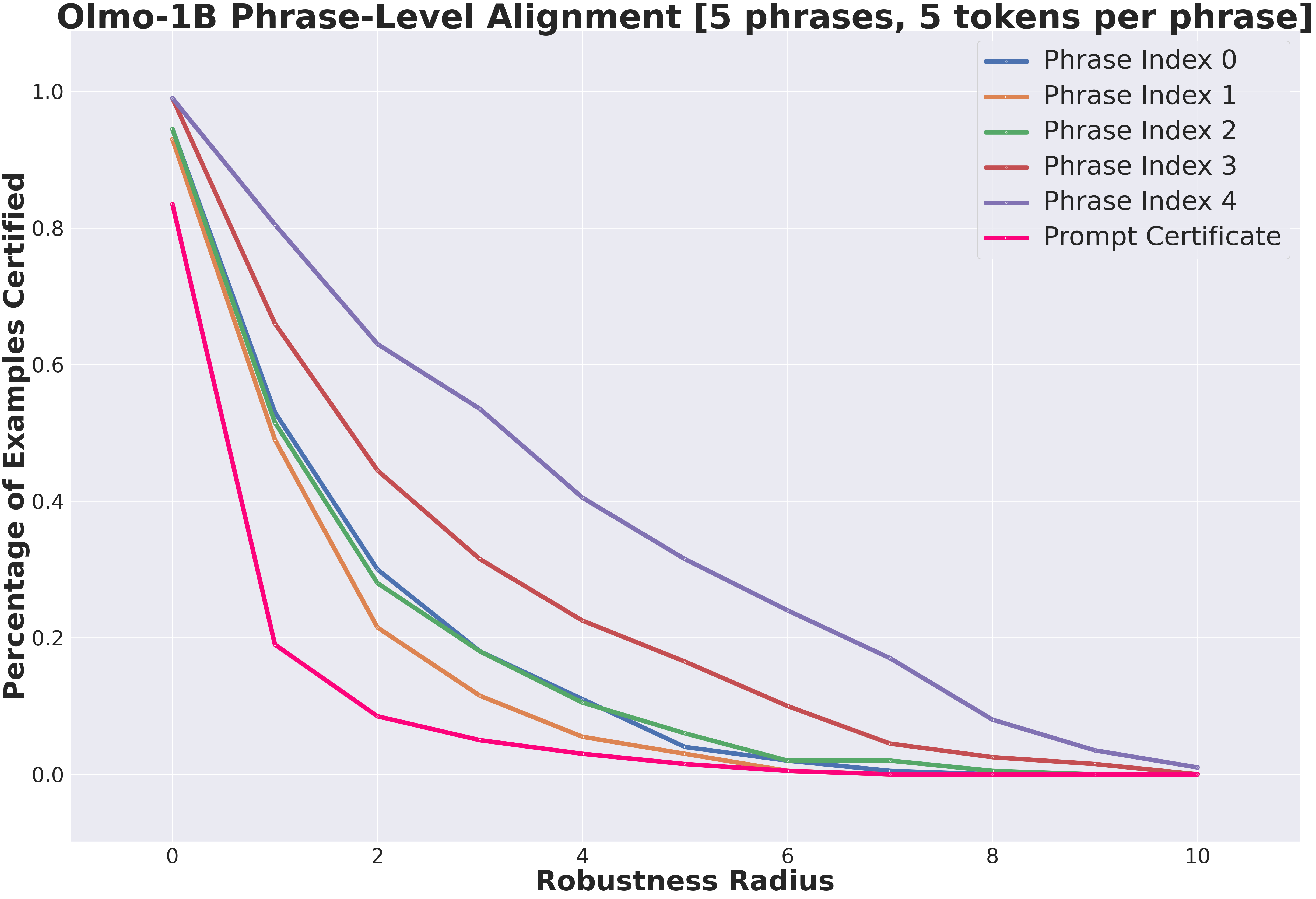}
     \end{subfigure}
     \hfill
     \begin{subfigure}[b]{0.45\textwidth}
         \centering
         \includegraphics[width=\textwidth]{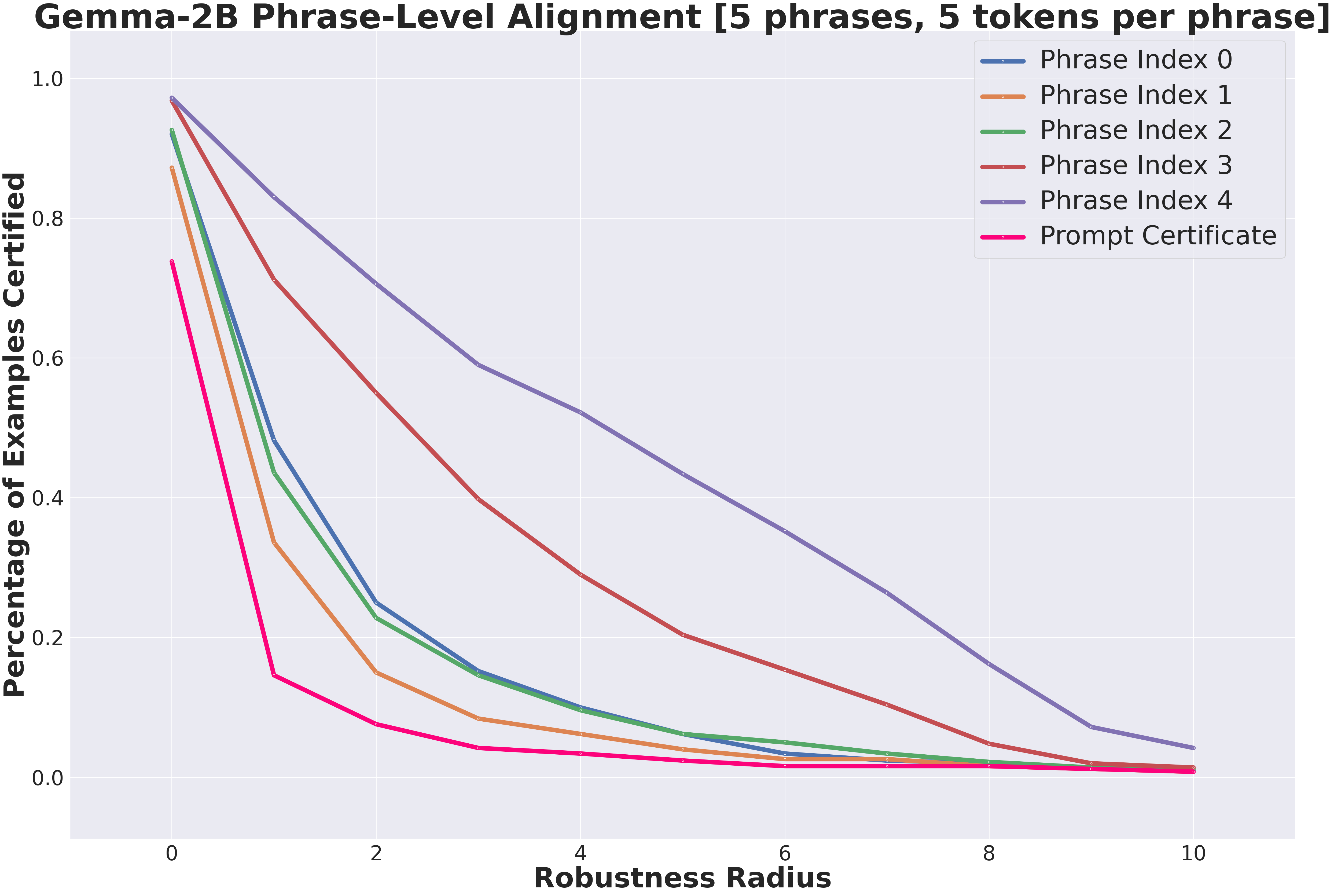}
     \end{subfigure}

     \vspace{10pt} 

     \begin{subfigure}[b]{0.45\textwidth}
         \centering
         \includegraphics[width=\textwidth]{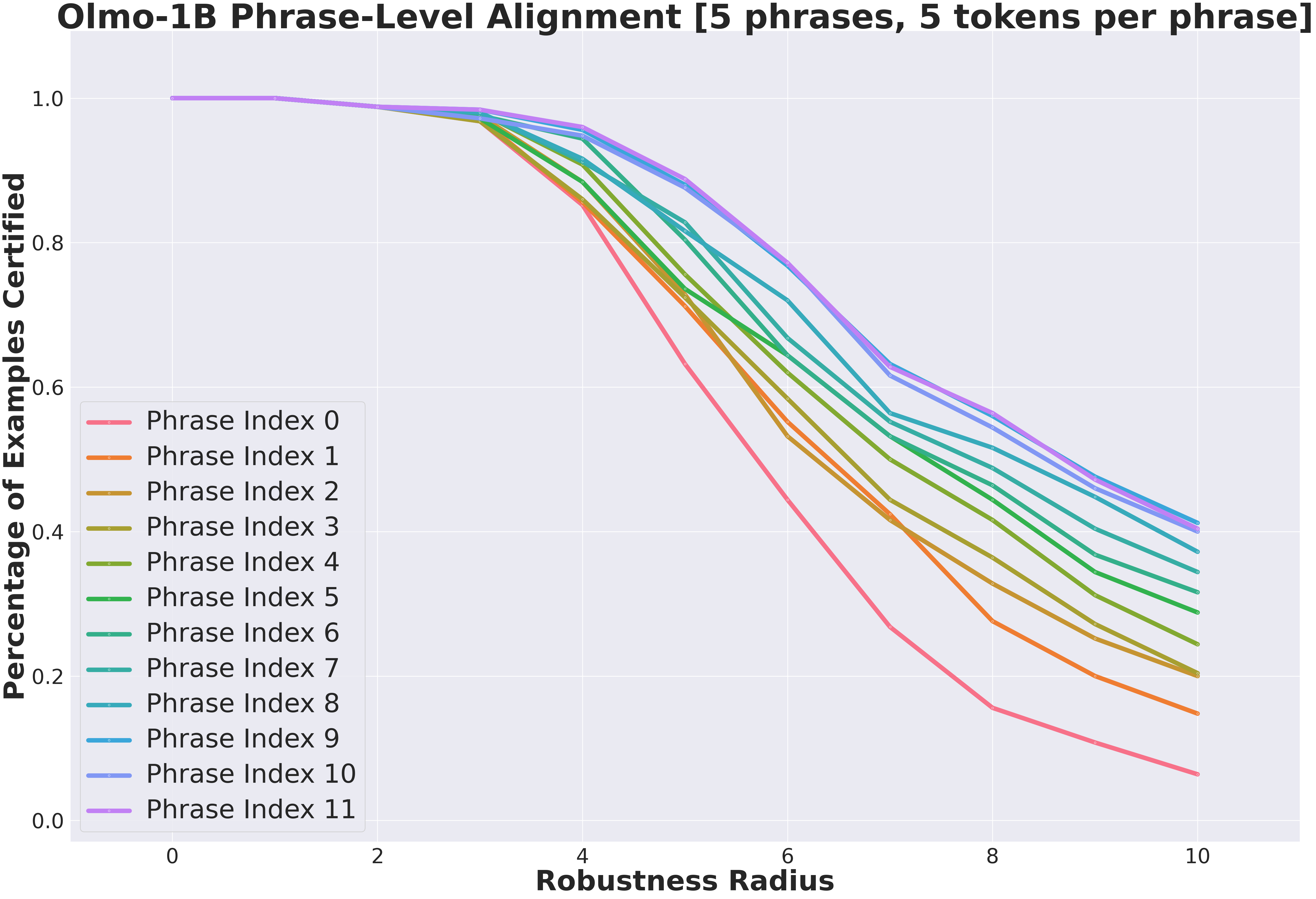}
     \end{subfigure}
     \hfill
     \begin{subfigure}[b]{0.45\textwidth}
         \centering
         \includegraphics[width=\textwidth]{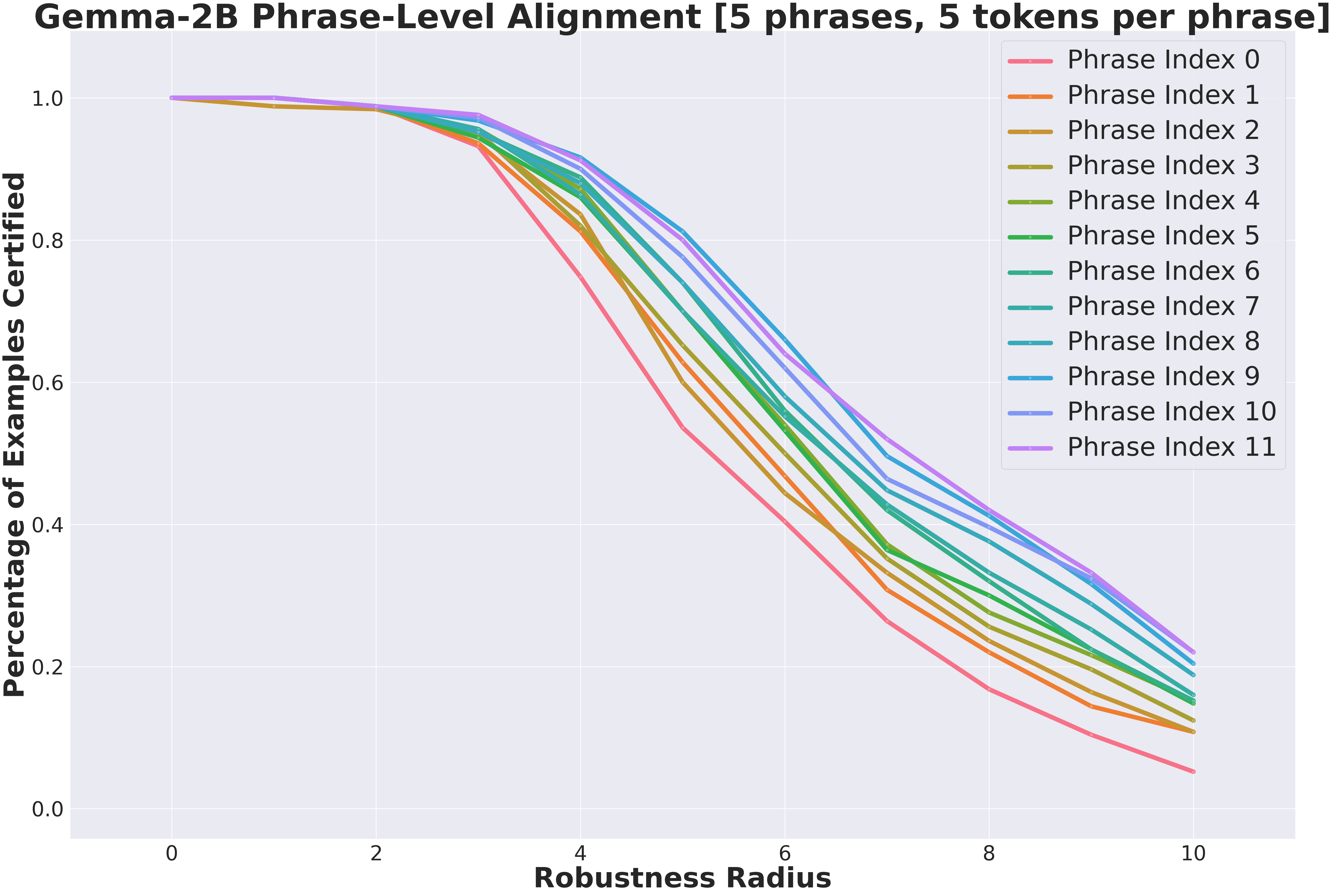}
     \end{subfigure}
     
     \caption{Phrase-level certified robustness as a function of robustness radius. Each curve represents the percentage of examples certified at different phrase indices ($m=5$ tokens per phrase). \textbf{Top row}: Certified stability for generations up to 25 tokens. \textbf{Bottom row}: Certified validity for generations up to 60 tokens. Robustness increases at higher phrase indices due to the collapse toward boilerplate or harmless responses. \textbf{Left column}: OLMo-1B. \textbf{Right column}: Gemma2-2B.}
     \label{fig:phrase_ablations}
\end{figure}

Phrase-level validity certification, on the other hand, is practically more appealing because of its improved computational complexity. Due to the fact that token-level harmful response avoidance is computed practically by reprompting each base language model at every token in the sentence (see Definition~\ref{def:ith-token-valid}), an ensemble of $M$ models voting on a sentence composed of $T$ tokens has an asymptotic time complexity of: 

\[
\mathcal{O}\left(M\times T \times C\right),
\]

where $C$ represents the inference cost of a single base model to a prompt. At the same time, doing phrase-level verification using $m$ tokens per phrase ($m=5$ in our case) yields an asymptotic time complexity of:

\[
\mathcal{O}\left(\frac{M\times T \times C}{m}\right).
\]

The bottom row of Figure~\ref{fig:phrase_ablations} reveals that the guarantees obtained are similar to those reported in \S\ref{:ssec:exp_stable_cert} for token-level validity, thus providing a way of maintaining robustness while simultaneously improving computational complexity. A similar behaviour to phrase-level stability certification is observed, except that the sequence to which the generation collapses is a harmless one, hence the improved robustness radius. 

\end{document}